\documentclass[11pt]{article}

\usepackage{fullpage}

\usepackage{graphicx} 
\usepackage{tikz}
\usetikzlibrary{positioning,arrows,arrows.meta}

\usepackage{amssymb, amsmath, amsthm, amsfonts, bbm}
\usepackage{url,graphicx,caption,subcaption}
\usepackage[style=alphabetic,natbib=true,maxbibnames=10]{biblatex}

\usepackage[colorlinks,linkcolor=blue!50!black,citecolor=blue!70!black,urlcolor=black]{hyperref}
\usepackage{boxedminipage}
\usepackage[capitalize,noabbrev]{cleveref}
\usepackage{array}
\usepackage{makecell}   
\usepackage{booktabs}
\usepackage{algorithm}
\usepackage{algpseudocode}
\usepackage{mathtools}

\newcolumntype{M}[1]{>{\centering\arraybackslash}m{#1}}
\usepackage{xspace}

\newtheorem*{theorem*}{Theorem}
\newtheorem{theorem}{Theorem}

\newtheorem{lemma}{Lemma}

\newtheorem{corollary}{Corollary}

\newtheorem{definition}{Definition}
\newtheorem{example}{Example}

\crefname{theorem}{Theorem}{Theorems}
\crefname{lemma}{Lemma}{Lemmas}
\crefname{corollary}{Corollary}{Corollaries}
\crefname{proposition}{Proposition}{Propositions}
\crefname{definition}{Definition}{Definitions}
\crefname{example}{Example}{Examples}
\crefname{remark}{Remark}{Remarks}
\crefname{assumption}{Assumption}{Assumptions}
\crefname{fact}{Fact}{Facts}
\crefname{result}{Result}{Results}
\crefname{problem}{Problem}{Problems}
\crefname{question}{Question}{Questions}

\newcommand{\R}{\mathbb{R}}  
\newcommand{\E}{\mathbb{E}}
\newcommand{\cX}{\mathcal{X}}

\newcommand{\cY}{\mathcal{Y}}
\newcommand{\cN}{\mathcal{N}}
\newcommand{\cO}{\mathcal{O}}
\newcommand{\cF}{\mathcal{F}}

\newcommand{\cP}{\mathcal{P}}

\newcommand{\cZ}{\mathcal{Z}}

\newcommand{\cD}{\mathcal{D}}
\newcommand{\eps}{\epsilon}

\newcommand{\Reg}{\mathrm{Regret}}
\newcommand{\regret}{\Reg}

\newcommand{\thetaPS}{\theta_{\mathrm{PS}}}
\newcommand{\muPS}{\mu_{\mathrm{PS}}}
\newcommand{\thetaPO}{\theta_{\mathrm{PO}}}

\renewcommand{\eps}{\varepsilon}

\newcommand{\defeq}{\coloneqq}
\DeclareMathOperator*{\argmin}{arg\,min}
\usepackage{xcolor}

\newcommand{\PR}{\mathrm{PR}}      
\newcommand{\DPR}{\mathrm{DPR}}

\def \endprf{\hfill {\vrule height6pt width6pt depth0pt}\medskip}
\renewenvironment{proof}{\noindent {\bf Proof} }{\endprf\par}

\begin{document}

\title{The Stability of Online Algorithms in Performative Prediction}
\author{Gabriele Farina  \\MIT \\ \url{gfarina@mit.edu} \and Juan Carlos Perdomo \\ NYU \\ \url
{j.perdomo.silva@nyu.edu}}

\date{\today}

\maketitle
\begin{abstract}
The use of algorithmic predictions in decision-making leads to a feedback loop where the models we deploy actively influence the data distributions we see, and later use to retrain on. This dynamic was formalized by \cite{perdomo2020performative} in their work on performative prediction. Our main result is an unconditional reduction showing that any no-regret algorithm deployed in performative settings converges to a (mixed) performatively stable equilibrium: a solution in which models actively shape data distributions in ways that their own predictions look optimal in hindsight. 
Prior to our work, all positive results in this area imposed strong restrictions on how models influenced distributions. By using a martingale argument and allowing randomization, we avoid any assumption on how populations respond to predictions and sidestep recent hardness results showing that deterministic stable models are in general PPAD-hard to compute. Lastly, on a more conceptual note, our connection sheds light on why common algorithms, like gradient descent, are naturally stabilizing and prevent runaway feedback loops. We hope our work enables future technical transfer of ideas between online optimization and performativity.
\end{abstract}

\begingroup
\setlength{\parskip}{0pt}
\tableofcontents
\endgroup

\section{Introduction}

Social prediction is an inherently dynamic affair. The predictions we make inform decisions that actively shape the data we see. These influenced outcomes are then fed back into the datasets used to train models, leading to a continual feedback loop between prediction systems and their encompassing social contexts. 

Once recognized, we start to see this impact, or \emph{performativity}, of social prediction everywhere. Predicting individuals' health outcomes influences their behavior and hence their realized outcomes. A bank's choice of credit risk model influences the features that people strategically report to the system. Link prediction algorithms in professional networks determine who we know, and hence what links will end up forming in the future. The reader can certainly come up with more examples. 

\citet*{perdomo2020performative} introduced a formal, learning-theoretic framework called \emph{performative prediction} to reason about learning algorithms in these dynamic environments. In their work, they identify a natural solution concept termed \emph{performative stability}. A model is performatively stable if it actively shapes data in a way that its own predictions look optimal in hindsight. 

The motivation behind stability is that it prevents a runaway feedback loop where learners repeatedly update their models because they are not risk-minimizing on the observed data.
In more detail, performativity often manifests as an unexpected distribution shift. Decision-makers often respond to this shift by \textit{retraining} their models on the most recent examples. Performatively stable models have the desirable property that they are \textit{fixed points} of retraining: if you deploy a performatively stable model and then retrain on its induced distributions, you get the same model back. There is no incentive to deviate. See \Cref{fig:stability} for a visual illustration. 

It seems intuitive that to guarantee stability, one needs tight control over whether the retraining has any fixed points. And indeed, there is now a rich literature on performativity showing that if $a)$ the learner's loss function is pointwise strongly convex and $b)$ populations are only weakly influenced by predictions, then retraining is a contraction (\emph{i.e.}, the feedback loop has Lipschitz constant $<$ 1 and hence converges to stability). While both of these conditions severely limit the regimes where we can guarantee stability, they are, in a precise sense, unavoidable. If losses are not strongly convex, or if predictions are strongly performative, then finding a single model that is performatively stable can be impossible (they may not exist) or PPAD-hard to compute \citep{anagnostides2026computational}.

However, in a fortunate twist, these difficulties can be easily avoided by simply allowing the learner to randomize across a small number of models. Our main contribution is an unconditional reduction from regret minimization to performative stability. We show that any learner that uniformly randomizes across the iterates produced by any no-regret algorithm is provably stable. This holds even when making \emph{no assumptions} about the distributional response: predictions can influence both features and outcomes in arbitrary ways. Some of the consequences of this result are that we establish the first positive results for finding stable points for losses such as squared loss or log-loss that are convex but not pointwise strongly convex, as well as the first polynomial sample complexity guarantees for repeated risk minimization. Furthermore, all of these results also carry over to the multiplayer setting (see \Cref{sec:multiplayer}), where the observed data is a function of many agents making predictions simultaneously. They also apply to the stateful setting (see \Cref{sec:stateful}) where the observed data is a function of the entire history of deployed models.

Taken together, our work significantly expands the space of problems for which one can efficiently achieve performative stability and highlights how classical supervised learning algorithms are naturally equilibrating when deployed in these more dynamic domains.  Looking forward, we hope our results will inspire future work at the intersection of performativity and online learning.

\begin{figure}[t]
    \centering
    \scalebox{.96}{\begin{tikzpicture}
        \node[rounded corners=.5mm,thick,draw,align=center,text width=2.6cm] (DATA) at (0,0) {\small Model influences\\ data $z_t \sim \cD(\theta_t)$};
        \node[rounded corners=.5mm,thick,draw,align=center,text width=2.6cm] (SAMPLE) at (-2.6,-2.5) {\small Learner samples\\ model $\theta_t \sim \mu_{\text{PS}}$};
        \node[rounded corners=.5mm,thick,draw,align=center,text width=2.0cm] (LOSS) at (2.3, -2.5) {\small Loss $\ell(z_t;\theta_t)$\\ is incurred};

        \draw[semithick,->,shorten >=1mm,shorten <=1mm] (SAMPLE) edge[bend left] (DATA);
        \draw[semithick,->,shorten >=1mm,shorten <=1mm] (DATA) edge[bend left] (LOSS);
        \draw[semithick,->,shorten >=1mm,shorten <=1mm] (LOSS) edge[] (SAMPLE);

        \node[fill=cyan!15, align=center,inner sep=2mm, rounded corners, text width=8.3cm,anchor=west] at (4.4, -1.2) {
        Informally, the \emph{mixture} $\mu_\text{PS}$ is $\epsilon$-performatively stable if for any other model $\theta'$, in expectation,
        \[
        \frac{1}{n}\sum_{t=1}^n \ell(z_t ;\theta_t) \leq  \frac{1}{n}\sum_{t=1}^n \ell(z_t ;\theta')
        \]
        That is, no other model has lower loss over the data induced by $\mu$.
        };
    \end{tikzpicture}}
    \caption{Visualizing the performative stability guarantee. A mixture $\mu$ over models $\theta$ is stable if it shapes the data in a way that its own predictions are risk minimizing. Our notion, \Cref{def:ps}, generalizes the prior notion where $\mu$ is a point mass on a single model.}
    \label{fig:stability}
\end{figure}

\subsection{Contributions}

To discuss our results, we give a brief overview of the performative prediction framework.  

Performative prediction extends supervised learning to account for the fact that deploying predictive models actively shapes the distribution of future data. More formally, if the learner makes predictions using a model $f_\theta$ where $\theta\in \R^d$ denotes a vector of parameters, then they see examples $z \sim \cD(\theta)$ drawn i.i.d. from the distribution $\cD(\theta)$ induced by the model.  $\cD(\cdot)$ is the distribution map, a function from model parameters $\theta$ to distributions over points in a set $\cZ$. Typically, we let $z=(x,y)$ consist of features $x \in \cX$ and outcomes $y \in \cY$. 
In this case, the distribution map $\cD(\cdot)$ captures how deploying models can influence both the marginal distribution over features and the conditional distribution over outcomes. Performative prediction therefore strictly subsumes the setting known as strategic classification \citep{hardt2016strategic}.

Given the dynamic effects of prediction, the following stability notion is one of the key desiderata within performativity:

\begin{definition}\label{def:ps}
Given a loss function $\ell(z;\theta)$, a model $\thetaPS$ is $\eps$-performatively stable if 
\begin{align*}
    \E_{z \sim \cD(\thetaPS)} \ell(z; \thetaPS) \leq \min_{\theta}  \E_{z \sim \cD(\thetaPS)} \ell(z; \theta) + \eps.
\end{align*}
More generally, we say that a mixture $\muPS$ over models $\theta$ is $\eps$-performatively stable if 
\begin{align*}
    \E_{\theta \sim \muPS} \E_{z \sim \cD(\theta)} \ell(z;\theta) \leq \min_{\theta' \in \Theta} \E_{\theta \sim \muPS} \E_{z \sim \cD(\theta)} \ell(z; \theta') + \eps.  
\end{align*}    
\end{definition}
Performatively stable models are the fixed points of the following retraining procedure,
\begin{align}
\label{eq:retraining}
    \theta_{t+1} = \argmin_{\theta} \E_{z \sim \cD(\theta_t)} \ell(z;\theta).
\end{align}
This is a natural dynamic that arises in recommender systems, strategic classification, and even traffic prediction, where new models are trained on data that was influenced by previous models.

Our main result shows that any online algorithm that guarantees sublinear external regret against the stochastic losses $\ell(z_t; \theta_t)$ will, in fact, converge to performative stability. This result holds regardless of the underlying distribution map $\cD(\cdot)$. See \Cref{fig:main_theorem} for a visual illustration.

\begin{figure}[t]
\centering
\scalebox{.95}{
\begin{tikzpicture}
    \node[draw,fill=black!2,very thick,rounded corners,inner xsep=1mm,inner ysep=2mm,text width=3.7cm,align=center] (OA) at (0,0) {Online algorithm\\ {\small\color{gray} (e.g., gradient descent)}};
    \node[draw,semithick,inner xsep=0mm,inner ysep=2mm,text width=3.6cm,align=center] (ZT) at (5.3,0) {Deploy $\theta_t$ and observe $z_t\sim\cD(\theta_t)$};

    \draw[->,semithick,shorten <=1mm, shorten >=1mm] (OA) -- node[pos=.5,above,text width=2cm,align=center]{Model\\$\theta_{t}$} +(ZT);
    \draw[->,semithick,rounded corners=.7mm,shorten <=1mm, shorten >=1mm] (ZT.east) -- ++(.5,0) -- ++(0,-2.5) --node[pos=0.5,align=center,above,text width=5cm]{Train with new loss\\ $\ell(z_t; \cdot)$} ++(-10.2,0) |- (OA);
    
    
    \node[fill=green!15, align=center,inner xsep=0mm, inner ysep=4mm,rounded corners,text width=1.15*5.7cm] at (11.5, -0.9) {\textbf{Main Theorem} (Informal):\\[2mm]For any $T \geq 0$ and $\cD(\cdot)$, the uniform distribution over the models $\theta_1, \dots, \theta_T$ produced by the online algorithm is $\mathrm{Regret}(T)/T$-stable.};
\end{tikzpicture}}
\caption{Our main result shows that any no-regret algorithm yields a sequence of models whose uniform mixture is performatively stable. The idea is simple and surprisingly strong. It sidesteps recent hardness results and establishes the first proofs of stability that do not require any continuity assumptions on $\cD(\cdot)$ or curvature assumptions on $\ell$.}
\label{fig:main_theorem}
\end{figure}

\begin{theorem}[Informal]
Let $\theta_1, \dots, \theta_T$ be models produced by an online algorithm that achieves no regret on the sequence of losses $\ell_t(\theta) = \ell(z_t; \theta)$ where $z_t \sim \cD(\theta_t)$ for $t=1$ to $T$. And, define $\mu$ to be the uniform distribution over $\theta_1, \dots, \theta_T$. Then, $\mu$ is $\Reg(T)/T$-performatively stable.
\end{theorem}

It is by now a celebrated fact that online algorithms have good generalization properties in supervised learning \citep{cesa2004generalization}. Here, we show that when data is no longer drawn i.i.d. from a fixed distribution, but is instead \emph{performative} and influenced by the chosen model, no-regret dynamics have the desirable property that they converge to a natural notion of stability. 

As we will discuss, this connection significantly expands the kinds of problems for which one can guarantee performative stability. As easy corollaries of this main result, we can show that by taking a mixture over their iterates, common no-regret algorithms such as gradient descent and retraining (follow the leader) converge to performative stability under substantially weaker assumptions on the loss function $\ell$ and the distribution map $\cD(\cdot)$ than those previously used in the performative prediction literature. Readers already familiar with prior work can jump to \Cref{table:comparison} for a summarized comparison.

Crucially, earlier analyses required that predictions have only a very weak, smoothly varying influence on the induced distribution. Formally, they assumed that $\cD(\cdot)$ is Lipschitz (see \Cref{def:sensitivity}). More than just a technical condition, this assumption ruled out stability in important domains such as medicine or education, where people take actions that influence outcomes by thresholding model predictions (see, e.g., \citet{dews} and \Cref{sec:preliminaries} for an example). 

Moreover, recent work by \citet*{anagnostides2026computational} showed that finding a single stable model even for Lipschitz $\cD(\cdot)$ and strongly convex and smooth $\ell$ is PPAD-complete.
By working with mixtures over models, we sidestep these complexity barriers, greatly expand the settings in which stability can be achieved efficiently, and provide a simple theoretical justification for why widely used learning practices prevent runaway feedback loops in social prediction.

\subsection{Related Work}\label{sec:related}

\paragraph{Performative prediction.} \citet*{perdomo2020performative} introduced the performative prediction framework and established that retraining, as described in \Cref{eq:retraining}, converged to performative stability if losses were smooth, strongly convex, and if the distribution map was sufficiently Lipschitz. Subsequent work \citep{mendler2020stochastic,drusvyatskiy2023stochastic} proved finite-sample results showing that various stochastic optimization algorithms, such as gradient descent, also converged to stability under similar assumptions on the loss and distribution map. Later, \citet*{taori2023data} and \citet*{mofakhami2023performative} analyzed convergence to stability under alternative Lipschitz conditions on $\cD(\cdot)$. 
Closely related to our work, \citet{perdomo2025revisiting} illustrated how online multi-calibration \citep{hkrr} algorithms converge to performative stability if outcomes $y$ are binary and predictions only influence the distribution over outcomes and not features. Furthermore, \citet{farina2026stabilizingperformativefeedbackloops} show how to find performatively stable mixtures in few deployments by applying recent algorithmic machinery for solving minimax problems.

In addition to these, there has been work analyzing performative stability in stateful settings \citep{brown2022performative,li2022state,ray2022decision,wood2021online} where the induced distribution is a function of the entire history of deployed models, as well as in multi-agent settings where data depends on the predictions made by multiple decision-makers \citep{li2022multi,piliouras2023multi,narang2023multiplayer}. All of these results also imposed strong Lipschitzness conditions on the distribution map $\cD(\cdot)$.

Lastly, there is a parallel line of work in this area that designs algorithms which converge to a different solution concept called \emph{performative optimality} \citep{miller2021outside,izzo2021learn,jagadeesan2022regret,kim2022making,pluginpp}. We discuss the differences between these two concepts in the next section and point the reader to \cite{hardt2023performative} for a broader survey of this area.

\begin{table}[t]
\centering
\begin{tikzpicture}[x=1.08cm,y=1cm]
    \tikzstyle{nn}=[,anchor=north west,minimum height=1.15cm,inner xsep=1.08*0.2cm, inner ysep=0,rounded corners=.07cm];
    \tikzstyle{c1}=[text width=1.08*2.5cm];
    \tikzstyle{c2}=[text width=1.08*2.0cm];
    \tikzstyle{c3}=[text width=1.08*2.5cm];
    \tikzstyle{c4}=[text width=1.08*2.5cm];
    \tikzstyle{c5}=[text width=1.08*3.2cm];
    \tikzstyle{good}=[fill=green!10];
    \tikzstyle{rbad}=[fill=red!10];
    \tikzstyle{bad}=[fill=red!10];
    \tikzstyle{low}=[text height=1.0*0.7cm];

    \draw[line width=.3mm] (0,-.8) -- +(15.2,0);
    \draw[               ] (0,-2.08) -- +(15.2,0);
    \draw[line width=.3mm] (0,-3.4) -- +(15.2,0);

    \foreach \x in {2.95,5.45,8.45,11.45} {
        \draw[line width=.1mm,black!20]
              (\x, 0) -- ++(0,-.7)
            ++(0,-.2) -- ++(0,-1.08)
            ++(0,-.2) -- ++(0,-1.1);
    }
    
    \node[nn,low,c1,align=right] at (0,.4) {Solution};

    \node[nn,c1,align=right] at (0,-.85) {Stable Model\\ $\theta_\text{PS}$};

    \node[nn,c1,align=right] at (0,-2.15) {Stable Mixture\\ $\mu_\text{PS}$ (this work)};

    \node[nn,low,c2,align=center] at (3,.4) {$\rho < 1$};

    \node[nn,c2,align=center] at (3,-.85) {$\cO(1/T)$\\ $\dagger$};

    \node[nn,good,c2,align=center] at (3,-2.15) {$\tilde{\cO}(1/T)$};

    \node[nn,low,c3,align=center] at (5.5,.4) {$\rho > 1 + \Omega(1)$};

    \node[nn,bad,c3,align=center] at (5.5,-.85) {{\small PPAD-Complete}\\ $\ddagger$};

    \node[nn,good,c3,align=center] at (5.5,-2.15) {$\tilde{\cO}(1/T)$};

    \node[nn,c4,align=center] at (8.5,.4) {Convex losses,\\arbitrary $\cD(\cdot)$};

    \node[nn,rbad,c4,align=center] at (8.5,-.85) {\small Not guaranteed\\ to exist};

    \node[nn,good,c4,align=center] at (8.5,-2.15) {${\cO}(1/\sqrt T)$};

    \node[nn,c5,align=center] at (11.5,.4) {Squared or log\\loss, arbitrary $\cD(\cdot)$};

    \node[nn,rbad,c5,align=center] at (11.5,-.85) {\small Not guaranteed\\ to exist};

    \node[nn,good,c5,align=center] at (11.5,-2.15) {$\tilde{\cO}(1/T)$};
\end{tikzpicture}
\caption{Comparison to prior work. We let $\rho = \eps\beta /\gamma$ denote the ``standard'' setting where $\ell$ is $\gamma$-strongly convex and $\beta$-smooth in $\theta$, and $\cD(\cdot)$ is $\eps$-Lipschitz. Our reduction establishes the first known stability guarantees for arbitrary distribution maps $\cD(\cdot)$, and weakly convex losses, while matching the rates in prior work up to log factors.
$\dagger$: \citet{mendler2020stochastic,drusvyatskiy2023stochastic}. $\ddagger$: \citet{anagnostides2026computational}.}
\label{table:comparison}
\end{table}

\paragraph{No-regret learning.} Our reduction places performative stability within a broader literature connecting regret minimization to other areas. In game theory, regret-minimizing dynamics are closely related to equilibrium concepts such as coarse correlated and correlated equilibria \citep{hart2000simple,dagan2023external,peng2024fast}, Nash equilibria in zero-sum games, and more general $\Phi$-equilibria \citep{gordon2008no}. These connections have been studied in a range of settings, including normal-form games, extensive-form games \citep{Farina22:Simple,farina2023polynomial}, harmonic games \citep{legacci2024no}, and concave games \citep{daskalakis2025efficient}. 
In a similar vein, calibrated forecasting can also be famously reduced to no-regret learning~\citep{foster1997calibrated,foster1998asymptotic}. Blackwell’s approachability theorem has been shown to be algorithmically equivalent to online linear optimization, yielding a tight reduction between approachability and no-regret learning \citep{abernethy2011blackwell}. 

In nonlinear optimization, recent reductions from nonsmooth, nonconvex optimization (in the sense of Goldstein stationarity) to online learning have led to optimal stochastic algorithms \citep{cutkosky2023optimal}.
In statistics and probability theory, sequential testing and anytime-valid inference can be reduced to no-regret algorithms, giving rise to e-processes and confidence sequences \citep{orabona2016coin,ramdas2023gametheoretic}. Deterministic regret inequalities have been shown to be equivalent to sharp martingale tail bounds and associated concentration inequalities \citep{rakhlin2017equivalence}.
Finally, multiplicative-weights and related online learning algorithms yield conceptually simple proofs of von Neumann's minimax theorem \citep{papadimitriou2008computing,farina2024polynomial,sigecom}, spectral sparsification in graph theory \citep{allen2015spectral}, and certain results in derandomization \citep{trevisan19chernoff}.
Our result contributes to this line of work by showing that performative stability can also be understood through the lens of no-regret learning.

\section{Preliminaries and Technical Background}

\label{sec:preliminaries}
Before presenting our main results, we provide some further context on definitions and prior results in performative prediction.
In addition to performative stability, there is a second main solution concept in performative prediction known as \emph{performative optimality}.

\begin{definition}
A model $\thetaPO$ is performatively optimal if,
\begin{align*}
    \E_{z \sim \cD(\thetaPO)} \ell(z; \thetaPO) \leq \min_{\theta}  \E_{z \sim \cD(\theta)} \ell(z; \theta).
\end{align*}
\end{definition}
Here, $\PR(\theta)=\E_{z \sim \cD(\theta)} \ell(z; \theta)$ is known as the \emph{performative risk} of the model specified by $\theta$. Optimal models minimize the performative risk. If we denote by $\DPR(\theta, \cD(\theta')) = \E_{z\sim \cD(\theta')} \ell(z ; \theta)$ the \emph{decoupled} performative risk, we can relate optimality and stability as follows,
\begin{align*}
   \text{\underline{Stability:}} \quad \PR(\thetaPS) \leq \min_{\theta} \DPR(\theta,\cD(\thetaPS)) ; \quad \qquad  \text{\underline{Optimality:}} \quad  \PR(\thetaPO) \leq \min_{\theta} \DPR(\theta,\cD(\theta)).
\end{align*}
Both solution concepts are generally distinct. \cite{perdomo2020performative} showed how stable models need not be performatively optimal and optimal models need not be stable. 

Given its connection to the retraining feedback loops we discussed in the introduction, the question of when and why various algorithms converge to stability has been a main focus of the performative prediction literature. The bulk of prior work analyzed convergence to performative stability under the following Lipschitz condition on the distribution map:
\begin{definition}
\label{def:sensitivity}
    A distribution map $\cD(\cdot)$ is $\eps$-sensitive (or Lipschitz) if for all $\theta, \theta'$
    \begin{align*}
        \mathcal{W}_1\big(\cD(\theta), \cD(\theta')\big)\leq \eps \cdot \|\theta - \theta'\|_2.
    \end{align*}
Here, $\mathcal{W}_1$ is the earth mover's or Wasserstein-1 distance between the distributions $\cD(\theta)$ and $\cD(\theta')$.
\end{definition}
Intuitively, the sensitivity condition formalizes the notion that small changes in predictions lead to small changes in the induced distribution. A canonical example where this assumption holds is in models of strategic classification \citep{hardt2016strategic}, where people manipulate their features to game a classifier. In particular, some works assume that if a linear classifier with parameters $\theta$ is deployed, individuals modify their features from $x$ to $x - \eps \theta$ in an effort to move to the positive side of the decision boundary. 
Using this sensitivity condition (\Cref{def:sensitivity}), prior work in this area established the following result.

\begin{theorem}[\cite{perdomo2020performative,mendler2020stochastic}]
\label{theorem:prior}
Assume that $\ell(z;\theta)$ is $\beta$-smooth, and $\gamma$-strongly convex in $\theta$ for every $z$. Furthermore, assume that the distribution map $\cD(\cdot)$ is $\eps$-sensitive  with $\rho=  \eps \beta /\gamma < 1$. Then, retraining, 
\begin{align*}
    \theta_{t+1} = \argmin_{\theta} \E_{z \sim \cD(\theta_t)} \ell(z; \theta)
\end{align*}
converges to a unique stable point $\thetaPS$ at a linear rate, $\|\theta_T - \thetaPS\| \leq \cO(\exp(-T))$. Additionally, if $\nabla_\theta \ell(z;\theta)$ is $\beta$-Lipschitz as a function of $z$ for every $\theta$, then stochastic gradient descent, 
\begin{align*}
    \theta_{t+1} = \theta_t - \eta_t \nabla_\theta \ell(z_t; \theta_t), \quad z_t \sim \cD(\theta_t),
\end{align*}
converges to a unique stable point at a polynomial rate, $\E\|\theta_T - \thetaPS\|^2 \leq \cO(1/T)$.
\end{theorem}

In other words, if the loss function is well-conditioned (\emph{pointwise} smooth and strongly convex) and the distribution map is Lipschitz with a small Lipschitz constant, then repeatedly retraining overcomes performativity and strategic behavior. While promising, the regime in which these results hold is limited. Prior work provides examples showing that retraining and gradient descent can diverge if losses are only convex but not strongly convex, or non-smooth, or if $\rho = \eps \beta / \gamma > 1$.\footnote{Note that the squared loss $(y -\theta^\top x)^2$ is not pointwise strongly convex.} More importantly, they illustrate how things can quickly destabilize if the distribution map $\cD(\cdot)$ is not continuous. 

Discontinuous $\cD(\cdot)$ commonly arise in practice whenever decision-makers choose actions by thresholding a prediction. To give a concrete example, public schools often assign interventions to students whose predicted probability of graduation, $f_\theta(x)$, falls below a cutoff. \cite{dews} empirically show how the conditional likelihood of graduation for public school students in Wisconsin has a discontinuous jump if $f_\theta(x)$ is below this cutoff: $$
\Pr\Big[y = 1 \mid f_\theta(x) > \tau\Big] \ll \Pr\Big[y = 1 \mid f_\theta(x) \le \tau\Big].$$ 

These discrete jumps are not just common in practice; they also pose a technical barrier for convergence analyses. To illustrate this, we recall the following example from \cite{perdomo2020performative}.

\begin{example}[\cite{perdomo2020performative}]
Fix $\ell(z;\theta)$ to be the squared loss $\ell(z;\theta) = (z-\theta)^2$ and assume that $z \sim \cD(\theta)$ is a point mass distribution where $z = 0$ if $\theta > 1/2$ and $z=1$ if $\theta \leq 1/2$ for $\theta  \in [0,1]$. The performatively optimal solution for this problem is $\thetaPO = 1/2$. However, there is no performatively stable point since the following function has no fixed points
\begin{align*}
    g(\theta) = \argmin_{\theta' \in [0,1]} \E_{z \sim \cD(\theta)} (z - \theta')^2.
\end{align*}
Furthermore, retraining as per \Cref{eq:retraining} will produce the infinite alternating sequence $0,1,0,1,$...
\end{example}
From this example, we see that performatively stable models are not guaranteed to exist in general. Moreover, recent work by \citet{anagnostides2026computational} illustrates how even if we do assume that stable models exist, it is PPAD-hard to compute a deterministic performatively stable point even if $\ell$ is $\gamma$-strongly convex, $\beta$-smooth, and $\cD(\cdot)$ is $\eps$-sensitive with $\rho = \beta \eps / \gamma > 1$. In short, it is provably impossible to find deterministic stable points in full generality.

In this work, we sidestep these negative results by allowing the learner to \emph{randomize} across models. This enables us to guarantee stability without requiring continuity assumptions on $\cD(\cdot)$. We summarize these comparisons to prior work in \Cref{table:comparison}.

Mixture models are not just necessary to guarantee stability; they are also practically well-motivated. To see this, it helps to go back to the Wisconsin dropout example. Assume that the department uses a mixture of models $\mu$. When deployed, students respond to the specific prediction $f_\theta(x)$ made on them, not the mixture $\mu$  ($\theta\sim \mu$ is the random model deployed at the time of their prediction).
The mixture is stable if, on average across all students, no model would have done any better, and the department has no incentive to retrain. That is, 
\begin{align*}
    \E_{\theta \sim \muPS} \E_{z \sim \cD(\theta)} \ell(z;\theta) \leq \min_{\theta'} \E_{\theta \sim \muPS} \E_{z \sim \cD(\theta)} \ell(z;\theta').
\end{align*}
This is the idea illustrated visually in \Cref{fig:stability}. We close this section by showing that, while not true in general, for certain problems performatively stable mixtures can also be performatively optimal. Furthermore, these mixtures are the natural limit of no-regret algorithms.

\addtocounter{example}{-1}
\begin{example}[Continued]
Let $\theta_1, \dots, \theta_T$ be the sequence of iterates produced by repeatedly retraining, that is, running the follow-the-leader algorithm, 
\begin{align*}
	\theta_{t} = \argmin_{\theta} \sum_{\tau=1}^{t-1} \ell(z_\tau; \theta) \text{ where } z_\tau \sim \cD(\theta_\tau).
\end{align*}
For this problem, we can check that $\theta_{t} = \frac{1}{t-1}\sum_{\tau=1}^{t-1} z_\tau = 1/2 - \frac{1}{2(t-1)} \cdot 1\{t \text{ is odd}\}$.

Moreover, the uniform mixture $\muPS$ over $\theta_1, \dots, \theta_T$ is $\widetilde{\cO}(1/T)$ performatively stable. And, for any $\alpha > 0$, $$\Pr_{\theta \sim \muPS}\bigl[\theta \in (1/2-\alpha,1/2]\bigr] \rightarrow 1$$ as $T \rightarrow \infty$. Hence, the limiting distribution is also performatively optimal since $\thetaPO=1/2$. 
\end{example}

\section{Performative Stability of Online Algorithms}

The convergence to stability in the previous example is not an accident, but rather a reflection of a more general result that we will now present. Before proving our main theorem, we recall the definition of a no-regret learning algorithm. 
\begin{definition}
Consider the online learning protocol, where at every time step $t$, the Learner selects a vector $\theta_t \in \Theta$, and then, having seen $\theta_t$, Nature selects a loss $\ell_t:\Theta \rightarrow \R$. 

An algorithm guarantees no regret if at every round $t$ it deterministically selects $\theta_t \in \Theta$ such that no matter the choice of function $\ell_t$,
\begin{align*}
    \sum_{t=1}^T \ell_t(\theta_t) - \min_{\theta \in \Theta} \sum_{t=1}^T \ell_t(\theta) \leq \regret_T \leq o(T)
\end{align*}
\end{definition}

\subsection{Main Result}

As mentioned in Section~\ref{sec:related}, online algorithms are popular not just because they have guarantees in adversarial settings, they also serve as a powerful template for algorithm design in more structured domains. In that vein, our main result shows that if the $\ell_t$ are equal to the stochastic losses $\ell(z_t; \cdot)$ then the uniform mixture over the iterates of any no-regret algorithm is performatively stable.  
\begin{theorem}
\label{theorem:main_result}
Let $\theta_1, \dots, \theta_T$ be the sequence of iterates produced over $T$ rounds by a deterministic online learning algorithm that guarantees sublinear regret on the sequence of losses $\ell_t(\theta) = \ell(z_t; \theta)$ where $z_t \sim \cD(\theta_t)$ and $\ell$ is bounded. 

Let $S=\{z_t\}_{t=1}^T$ be the dataset of $T$ draws from $\cD(\cdot)$ and $\mu$ be the uniform distribution over $\theta_1, \dots, \theta_T$. Then, 
\begin{align*}
    \underset{S,\theta \sim \mu, z \sim \cD(\theta)}{\E} \ell(z;\theta) \leq \min_{\theta' \in \Theta} \underset{S,\theta \sim \mu, z \sim \cD(\theta)}{\E}\ell(z; \theta') + \frac{\regret_T}{T}. 
\end{align*}
Hence, the mixture $\mu$ is $\regret_T/ T$-performatively stable.

Furthermore, if $\ell$ is $L$-Lipschitz such that for any $z$, $|\ell(z;\theta)-\ell(z;\theta')| \leq L \cdot\|\theta-\theta'\|_2$,  and uniformly bounded by $B$ such that $\sup_{z,\theta} |\ell(z ; \theta)| \leq B$, then with probability $1-\delta$ over $S$, 
\begin{align*}
 \underset{\theta \sim \mu, z \sim \cD(\theta)}{\E} \ell(z;\theta) \leq \min_{\theta' \in \Theta} \underset{\theta \sim \mu, z \sim \cD(\theta)}{\E}  \ell(z; \theta') + (1 + L)\frac{\regret_T}{T} + 2B \sqrt{\frac{2\bigl(\log |\cN_\eta| + \log(1/\delta)\bigr)}{T}}
\end{align*}
where $\cN_\eta$ is an $\eta$-net over the set $\Theta$ with $\eta = \regret_T/T$.
\end{theorem}
\begin{proof}
Let $\theta_\star$ be any vector in $\Theta$. The goal is to upper bound,
\begin{align*}
    \E_{\theta \sim \mu, z \sim \cD(\theta)}\left[\ell(z; \theta) - \ell(z; \theta_\star) \right].
\end{align*}
We start out by expanding the first expectation over $\mu$. Since $\mu$ is chosen to be the uniform distribution over $\theta_1, \dots, \theta_T$, this expression then becomes:
\begin{align*}
    \frac{1}{T}\sum_{t=1}^T\E_{z \sim \cD(\theta_t)}\Big[\ell(z; \theta_t) - \ell(z; \theta_\star) \Big].
\end{align*}
Next, we add and subtract $\ell(z_t;\theta_t) - \ell(z_t; \theta_\star)$ where $z_t \sim \cD(\theta_t)$ for every $t$. Rearranging, the expression becomes
\begin{align*}
   \underbrace{\frac{1}{T}\sum_{t=1}^T\E_{z \sim \cD(\theta_t)}\Big[\ell(z; \theta_t) - \ell(z; \theta_\star) - (\ell(z_t;\theta_t)- \ell(z_t; \theta_\star))\Big]}_{\text{small by a martingale argument}} +  \underbrace{\frac{1}{T}\sum_{t=1}^T\Big[\ell(z_t; \theta_t) - \ell(z_t; \theta_\star)\Big]}_{\text{small by regret}}.
\end{align*}
In particular, if we define $V_t= \E_{z \sim \cD(\theta_t)}\Big[\ell(z; \theta_t) - \ell(z; \theta_\star)\Big] - [\ell(z_t;\theta_t)- \ell(z_t; \theta_\star)]$, then 
\begin{align*}
 \frac{1}{T}\sum_{t=1}^T \E_{z \sim \cD(\theta_t)}\left[\ell(z; \theta_t) - \ell(z; \theta_\star) - (\ell(z_t;\theta_t)- \ell(z_t; \theta_\star))\right] = \frac{1}{T}\sum_{t=1}^T V_t.
\end{align*} 
Here, $V_1, \dots, V_T$ is a martingale difference sequence with
respect to the filtration $\{\mathcal{F}_t\}_{t\geq 0}$ defined by $\cF_0 \defeq \sigma(\theta_1)$
\[
    \cF_t \defeq \sigma(\theta_1, z_1, \theta_2, z_2, \ldots,
    \theta_t, z_t, \theta_{t+1}) \quad \text{for } t \geq 1,
\]
where $\sigma(\cdot)$ denotes the $\sigma$-algebra generated by the random variables. By
construction, $\theta_t$ is measurable with respect to $\mathcal{F}_{t-1}$ (we assumed the online algorithm is deterministic), and
conditional on $\mathcal{F}_{t-1}$ the sample $z_t$ is an
independent draw from $\cD(\theta_t)$. Hence, $V_t$ is
$\mathcal{F}_t$-measurable and
\[
    \E\Big[V_t \mid \mathcal{F}_{t-1}\Big]
    = \E_{z\sim \cD(\theta_t)}\Big[\ell(z;\theta_t)-\ell(z;\theta_\star)\Big]
    - \E\Big[\ell(z_t;\theta_t)-\ell(z_t;\theta_\star) \,\bigm|\, \mathcal{F}_{t-1}\Big]
    = 0.
\]
Since losses are bounded, $\E[|V_t|]<\infty$, and therefore
$\E_S[\sum_{t=1}^T V_t] = 0$. Since $\theta_\star$ was arbitrary, this proves the first statement. 

To show the second statement, we can just union bound over an $\eta$-net. In particular, let $\cN_\eta$ be an $\eta$-net of $\Theta$. That is, define $\cN_\eta\subset \Theta$ to be the set consisting of the minimum number of points in $\Theta$ such that for every $\theta' \in \Theta$ there exists $\theta \in \cN_\eta$ such that $\|\theta-\theta'\|_2\leq \eta$.

If we redo the same argument as before with respect to each $\theta\in \cN_\eta$ and take a union bound over $ \cN_\eta$, we obtain that with probability at least $1-\delta$,
\begin{align}
\label{eq:net-uniform-bound}
\sup_{\theta \in \cN_\eta}
\frac{1}{T}\sum_{t=1}^T V_t(\theta)
\leq
2B \sqrt{\frac{2\bigl(\log |\cN_\eta| + \log(1/\delta)\bigr)}{T}}.
\end{align}
To complete the argument, we need to relate the loss of an arbitrary $\theta$ to that of one in $\cN_\eta$. By Lipschitzness of the loss, for any $\theta_\star$, there exists a $\bar\theta \in \cN_\eta$ such that:
\begin{align*}
    \frac{1}{T}\sum_{t=1}^T\E_{z \sim \cD(\theta_t)}\left[\ell(z; \theta_t) - \ell(z; \theta_\star) \right] &\leq \frac{1}{T}\sum_{t=1}^T\E_{z \sim \cD(\theta_t)}\left[\ell(z; \theta_t) - \ell(z; \bar\theta) \right] + \frac{1}{T}\sum_{t=1}^T\E_{z \sim \cD(\theta_t)}\left[\ell(z; \bar\theta) - \ell(z; \theta_\star) \right]  \\ 
    & \leq \frac{\regret_T}{T} + 2B \sqrt{\frac{2\bigl(\log |\cN_\eta| + \log(1/\delta)\bigr)}{T}} + \eta L.
\end{align*}
Setting $\eta$ to $\regret_T / T$ finishes the proof.
\end{proof}

On a technical level, the proof is an online-to-batch conversion that relies on a martingale argument. The main difference relative to standard online-to-batch analyses is that at every time step, we draw samples $z_t$ from $\cD(\theta_t)$ instead of a fixed distribution as in supervised learning. This makes critical use of the fact that online algorithms can guarantee low regret against adaptive adversaries that choose $\ell_t$ as a function of $\theta_t$. This allows us to set $\ell_t(\theta) = \ell(z_t; \theta)$ where $z_t \sim \cD(\theta_t)$ is drawn from the distribution induced by the iterate at time $t$. We also made use of the fact that online algorithms deterministically guarantee low regret against any comparator over the realized sequence of losses.

The statement is non-asymptotic (it holds for all $T\geq1$). It also establishes stability both in expectation and with high probability over the samples $S$ drawn from $\cD(\cdot)$. To the best of our knowledge, this is the first high-probability result for finding stable points. 

\paragraph{Comparison to prior work.} Prior work in the performative prediction literature has also explored connections to online learning. \citet{drusvyatskiy2023stochastic} show that under the assumptions of \Cref{theorem:prior}, namely that losses are (pointwise) $\beta$-smooth, $\gamma$-strongly-convex, and $\cD(\cdot)$ is $\eps$-sensitive with $\eps < \gamma/\beta$, then a single model equal to the average of the iterates $\bar{\theta} = \frac{1}{T}\sum_{t=1}^T \theta_t$ produced by an online algorithm is performatively stable. We on the other hand show that the mixture over iterates is performatively stable, but do not require smoothness, strong convexity, or that $\cD(\cdot)$ is continuous (much less Lipschitz). 

Perhaps the most closely related work to ours is \citet{perdomo2025revisiting}. He shows that in the \emph{outcome} performative setting where features $x$ are sampled from a fixed distribution $\cD_x$, and a binary outcome $y$ is sampled from a (possibly discontinuous) performative distribution $\cD_y(x,f_\theta(x))$ that depends on the learner's forecast $f_\theta(x)$, any algorithm that guarantees online multicalibration with respect to a particular set of functions can be converted to a (randomized) performatively stable model if $\ell$ is equal to the squared loss. By anchoring performativity on the model $\theta$ instead of the specific prediction $f_\theta(x)$, we achieve stability for general distribution maps $\cD(\cdot)$ where $\theta$ influences the \emph{joint} distribution over $z=(x,y)$ instead of just the distribution over $y$. Furthermore, our result holds for any loss where one can guarantee low regret and general outcomes $y$.

\subsection{Implications}

\Cref{theorem:main_result} is a general reduction showing we can extract a performatively stable mixture from the iterates produced by any no-regret learning algorithm. Here, we instantiate this reduction for specific choices of loss functions and algorithms, illustrating how \Cref{theorem:main_result} can be used to show that popular procedures are indeed stable in previously unknown settings. While prior work showed convergence to stability for different algorithms using tailored, bespoke arguments, our work provides a simple, unifying analysis that leverages known regret bounds to show convergence to stability. Most importantly, our results hold for arbitrary distribution maps $\cD(\cdot)$, significantly expanding the domains where we can guarantee stability.

Throughout this section, we will assume that losses have gradients that are uniformly bounded by $G$ and that the compact, convex domain $\Theta$ has diameter bounded by $D$:
\begin{align*}
	\sup_{z \in \cZ,\theta \in \Theta} \| \nabla_\theta \ell(z; \theta) \|_2 \leq G, \quad \sup_{\theta,\theta' \in \Theta} \|\theta - \theta'\|_2 \leq D.
\end{align*}

\subsubsection{Finite-Sample Guarantees for Retraining} Retraining, as defined in \Cref{eq:retraining}, is the most commonly studied approach to finding a performatively stable solution. All analyses of this procedure $(a)$ worked in the standard model with $\rho < 1$, and $(b)$ considered only an idealized procedure where the learner finds the risk-minimizing solution at the population level at every time step. The finite sample results of \citet{perdomo2020performative} are exponential in the problem dimension. 

However, retraining is exactly equal to the classical follow-the-leader algorithm. Using standard regret bounds, along with \Cref{theorem:main_result}, we can establish the first proof that smoothed retraining (where the learner retrains and then randomizes over the iterates) efficiently converges to stability in finite samples without making any assumptions on $\cD(\cdot)$. We recall the following lemma:

\begin{lemma}[\cite{shalev2008mind,hazan2007logarithmic,abernethy2008optimal}]
Assume that $\ell_1, \dots, \ell_T$ is a sequence of functions that are all $\gamma$-strongly convex. If $\theta_t$ are chosen according to the follow-the-leader algorithm,
\begin{align}
\label{eq:ftl}
    \theta_t \in \argmin_{\theta} \sum_{j=1}^{t-1} \ell_j(\theta),
\end{align}
then the iterates satisfy:
    $\sum_{t=1}^T \ell_t(\theta_t) - \min_\theta \sum_{t=1}^T \ell_t(\theta) \leq \frac{G^2}{2\gamma}(\log(T) + 1).$ 
\end{lemma}

Using this result, we can establish the following convergence for repeated retraining:
\begin{corollary}
Let $\ell$ be a $\gamma$-strongly convex loss with gradients bounded by $G$ over the domain $\Theta$. Assume that the learner runs the following repeated retraining procedure, 
\begin{align*}
    \theta_t \in \argmin_{\theta} \sum_{j=1}^{t-1} \ell(z_j; \theta) \text{ where } z_j \sim \cD(\theta_j).
\end{align*}
Then, in expectation over the samples $S = \{z_t\}_{t=1}^T$ drawn from $\cD(\cdot)$, the uniform mixture over $\theta_1, \dots, \theta_T$ is performatively stable:
\begin{align*}
      \underset{S,\theta \sim \mu, z \sim \cD(\theta)}{\E} \ell(z;\theta) \leq \min_{\theta' \in \Theta}     \underset{S,\theta \sim \mu, z \sim \cD(\theta)}{\E} \ell(z; \theta') + \frac{G^2}{2\gamma} \cdot\frac{(\log(T)+1)}{T}. 
\end{align*}
\end{corollary}

This result makes two adjustments relative to the standard definition of retraining as per \Cref{eq:retraining}. First, \Cref{eq:retraining} defines $\theta_t$ to be the population risk minimizer on $\cD(\theta_{t-1})$. Here, at every timestep, the learner finds the risk-minimizing solution over the \textit{entire} history of samples $z_t$ where each $z_t$ is drawn from a \emph{different} distribution. Furthermore, at the end, the learner randomizes across models instead of deploying the final iterate $\theta_T$. As discussed in \Cref{sec:preliminaries}, this modification is necessary both to avoid the hardness (or non-existence) of deterministic performatively stable models and to guarantee stability while making no assumptions on $\cD(\cdot)$.
    
\subsubsection{Improved Guarantees for Gradient Descent} Apart from retraining, prior work also studied whether the stochastic gradient method also converges to stability. As outlined previously, \citet{mendler2020stochastic} and \citet{drusvyatskiy2023stochastic} proved that gradient descent converges to stability if the loss is again smooth, strongly convex, and $\cD(\cdot)$ is Lipschitz with $\rho<1$. Using our reduction, we can remove any assumptions on $\cD(\cdot)$ and establish the first performative stability guarantees of any algorithm for losses that are non-smooth or just convex but not strongly convex. To do this, we recall the standard regret bounds for online gradient descent:
\begin{lemma}[\cite{zinkevich2003online,hazan2007logarithmic}]
Assume that $\ell_1,\dots,\ell_T$ is a sequence of convex functions. If $\theta_t$ are chosen using online gradient descent 
\begin{align*}
    \theta_{t+1} = \mathrm{Proj}_{\Theta} \left[ \theta_t - \eta_t \nabla \ell_t(\theta_t) \right], \text{ where } \eta_t = \frac{D}{G\sqrt{t}},
\end{align*}
then the iterates satisfy the following regret bound: 
\begin{align*}
     \sum_{t=1}^T \ell_t(\theta_t) - \min_\theta \sum_{t=1}^T \ell_t(\theta) \leq \frac{3}{2}  GD\sqrt{T}.
\end{align*}
If losses $\ell_t$ are additionally $\gamma$-strongly convex and we set $\eta_t=1/(\gamma t)$ then, 
\begin{align*}
     \sum_{t=1}^T \ell_t(\theta_t) - \min_\theta \sum_{t=1}^T \ell_t(\theta) \leq \frac{G^2}{2\gamma}(1 + \log T).
\end{align*}
\end{lemma}

Tying this lemma together with \Cref{theorem:main_result}, we get the following corollary:

\begin{corollary}
Assume that the learner runs the following (greedy) gradient descent algorithm,
\begin{align*}
    \theta_{t+1} = \mathrm{Proj}_{\Theta} \big[\theta_t - \eta_t \nabla_\theta \ell(z_t; \theta_t) \big], \text{ where } \eta_t = \frac{D}{G\sqrt{t}}, \; z_t \sim \cD(\theta_t),
\end{align*}
and that the loss $\ell$ has gradients bounded by $G$ over the domain $\Theta$. If the losses are convex, $\eta_t = D/(G\sqrt{t})$, then the uniform distribution over $\theta_1, \dots, \theta_T$ satisfies,
\begin{align*}
       \underset{S,\theta \sim \mu, z \sim \cD(\theta)}{\E} \ell(z;\theta) \leq \min_{\theta' \in \Theta}     \underset{S,\theta \sim \mu, z \sim \cD(\theta)}{\E} \ell(z; \theta') + \frac{3}{2} \cdot \frac{GD}{\sqrt{T}},
\end{align*}
where $S=\{z_t\}_{t=1}^T$. If the losses are $\gamma$-strongly convex, and $\eta_t= 1/(\gamma t)$, then 
\begin{align*}
       \underset{S,\theta \sim \mu, z \sim \cD(\theta)}{\E}\ell(z;\theta) \leq \min_{\theta' \in \Theta}     \underset{S,\theta \sim \mu, z \sim \cD(\theta)}{\E} \ell(z; \theta') + \frac{G^2}{2\gamma} \cdot \frac{1+ \log(T)}{T}.
\end{align*}
\end{corollary}
Ignoring quantities that depend on the diameter of the space $\Theta$, the bounds we get for gradient descent in the strongly-convex case match those from \citet{mendler2020stochastic} and \citet{drusvyatskiy2023stochastic} up to a $\log(T)$ factor. However, we achieve this guarantee for general distribution maps $\cD(\cdot)$, rather than just for Lipschitz $\cD(\cdot)$. This result holds for what \citet{mendler2020stochastic} refer to as \emph{greedy} gradient descent where the learner performs a model update and redeploys at every time step. 

It is an interesting direction for future work to understand if by leveraging other online algorithms (e.g., the primal-dual analysis in \cite{shalev2008mind}) one could also prove that \emph{lazy} versions of gradient descent where the learner performs many updates before redeployment also lead to stability, while reducing the total number of deployed models. Another interesting question is whether our online-to-batch analysis can be sharpened to yield a $\widetilde{\cO}(1/T)$ rate with high probability.

\subsubsection{Guaranteeing Performative Stability in New Regimes} As a final corollary, we highlight how the reduction to online learning can be used to guarantee stability in new domains (or in old domains at faster rates). In particular, to get a $\widetilde{\cO}(1/T)$ rate for stability, prior work required that the losses $\ell$ were \emph{pointwise} strongly convex. That is, for every $z$, $\ell(z;\theta)$ is a strongly convex function in $\theta$. If we don't add a regularizer, this fails for the two most common losses in machine learning: log loss $ \log(1 + \exp(-y \cdot \theta^\top x))$ and squared loss $(y-\theta^\top x)^2$. 

These losses are, however, exp-concave and one can achieve $\log(T)$ regret for these problems using second-order methods. Using this connection, we can establish the first positive results for performative stability in these common settings (and with fast rates).
\begin{corollary}
Let $\ell_1, \dots, \ell_T$ be a sequence of losses that are $\alpha$-exp-concave for $\alpha > 0$. That is, for every $\theta \in \R^d$, $\nabla^2 \ell_t(\theta) - \alpha \nabla \ell_t(\theta) \nabla \ell_t(\theta)^\top$ is a positive semi-definite matrix. There exists an efficient deterministic online algorithm, the online Newton step algorithm, that guarantees that for $T\geq 4$: 
\begin{align*}
     \sum_{t=1}^T \ell_t(\theta_t) - \min_\theta \sum_{t=1}^T \ell_t(\theta) \leq 2\left(\frac{1}{\alpha} + GD\right) \cdot d \cdot \log(T).
\end{align*}
Consequently, if the learner produces iterates $\theta_t$ by running online Newton step on the losses $\ell(z_t;\theta)$ for $z_t \sim \cD(\theta_t)$, the uniform distribution over iterates $\theta_1,\dots, \theta_T$ is performatively stable:
\begin{align*}
    \underset{S,\theta \sim \mu, z \sim \cD(\theta)}{\E} \ell(z;\theta) \leq \min_{\theta' \in \Theta}  \underset{S,\theta \sim \mu, z \sim \cD(\theta)}{\E} \ell(z; \theta') + 2\left(\frac{1}{\alpha} + GD\right) \cdot \frac{d \cdot \log T}{T}.
\end{align*}
where $S=\{z_t\}_{t=1}^T$ denotes the dataset drawn from $\cD(\cdot)$.
\end{corollary}
We point the reader to the excellent monograph \cite{hazan2016introduction} for details on the online Newton step algorithm. The corollary follows again by tying together known regret bounds \citep{hazan2007logarithmic} with \Cref{theorem:main_result}. 

Lastly, we note that we do not intend this section to be an exhaustive list of corollaries. Online optimization is by now a very rich and mature field with sophisticated algorithms that in many ways improve on the ones we present here (in particular on their dimension, and $G,D$ dependence). We chose these due to their simplicity and to highlight some of their consequences for performativity.
Furthermore, in these corollaries we take expectation over the samples in $S$ for simplicity, but we could have stated the high probability version by using the second part of \Cref{theorem:main_result}. 

\section{Applications to Multiplayer Performative Prediction}
\label{sec:multiplayer}

Our results so far considered the standard performative prediction setting with a single
learner. However, many prediction systems operate in shared environments with multiple
decision-makers, each deploying their own models. In traffic prediction, for instance,
commute times are jointly shaped by the routing recommendations of several navigation
apps. In digital advertising, user behavior is influenced by the combined choices of
platforms competing for attention. This dynamic was formalized by
\citet*{narang2023multiplayer} under the name of \emph{multiplayer} performative
prediction. As we now show, our reduction extends to this setting with essentially the
same proof.

\def\pl#1{^{#1}}
\def\pli{\pl{i}}
\def\pln{\pl{n}}

In the multiplayer problem, there are $n$ players. Each player $i \in [n]$ deploys a
model $\theta\pli \in \Theta\pli \subseteq \R^{d\pli}$, and the data observed by player $i$ is
drawn from a distribution that depends on \emph{all} of the deployed models. That is, $z\pli \sim \cD\pli(\theta\pl1, \dots, \theta\pln)$ where $\cD\pli(\cdot)$ is the distribution map for player $i$ which depends on the choice of all $n$ players. Furthermore, each player has their own loss $\ell\pli(z\pli; \theta\pli)$. The natural generalization of
performative stability to this setting is the following.

\begin{definition}[Multiplayer performative stability]
\label{def:multi-ps}
A joint distribution $\mu$ over profiles $(\theta\pl1, \dots, \theta\pln) \in \Theta\pl1
\times \cdots \times \Theta\pln$ is $\eps$-performatively stable if for every player
$i \in [n]$,
\begin{align*}
    \E_{(\theta\pl1, \dots, \theta\pln) \sim \mu} \E_{z\pli \sim \cD\pli(\theta\pl1, \dots,
    \theta\pln)} \ell\pli(z\pli; \theta\pli) \leq \min_{\hat\theta\pli \in \Theta\pli}
    \E_{(\theta\pl1, \dots, \theta\pln) \sim \mu} \E_{z\pli \sim \cD\pli(\theta\pl1, \dots,
    \theta\pln)} \ell\pli(z\pli; \hat\theta\pli) + \eps.
\end{align*}
\end{definition}
That is, no player can unilaterally reduce their expected loss under $\mu$ by switching
to a different model. Note that $\mu$ is a distribution over joint profiles, not a
product of independent marginals. This is what allows us to capture the correlation
between players' iterates that naturally arises from their shared interaction history.

Our main result extends with essentially the same proof: if every player runs a
no-regret algorithm on their observed losses, then the uniform distribution over the
joint trajectory is performatively stable.

\begin{theorem}
\label{thm:multi}
Suppose each player $i \in [n]$ produces iterates $\theta\pli_1, \dots, \theta\pli_T$ using
a deterministic online algorithm that achieves sublinear regret $\Reg\pli(T)$ against the sequence of
losses $\ell\pli_t(\theta\pli) = \ell\pli(z\pli_t; \theta\pli)$, where $z\pli_t \sim \cD\pli(\theta\pl1_t,
\dots, \theta\pln_t)$ and each $\ell\pli$ is bounded. 

Let $\mu$ be the uniform distribution
over the joint profiles $\{(\theta\pl1_t, \dots, \theta\pln_t)\}_{t=1}^T$. Then, in
expectation over the draws from each $\cD\pli(\cdot)$, the mixture $\mu$ is $(\max_i \Reg\pli(T)/T)$-performatively stable.
\end{theorem}

\begin{proof}
Fix a player $i \in [n]$ and an arbitrary comparator $\theta\pli_\star \in \Theta\pli$. For
brevity, let $\theta_t = (\theta\pl1_t, \dots, \theta\pln_t)$ and $z_t = (z\pl1_t, \dots,
z\pln_t)$. The goal is to upper bound
\begin{align*}
    \E_{(\theta\pl1, \dots, \theta\pln) \sim \mu} \E_{z\pli \sim \cD\pli(\theta\pl1, \dots,
    \theta\pln)} \big[\ell\pli(z\pli; \theta\pli) - \ell\pli(z\pli; \theta\pli_\star)\big]
    = \frac{1}{T}\sum_{t=1}^T \E_{z\pli \sim \cD\pli(\theta_t)}\big[\ell\pli(z\pli; \theta\pli_t)
    - \ell\pli(z\pli; \theta\pli_\star)\big],
\end{align*}
where the equality follows because $\mu$ is the uniform distribution over the joint
profiles. Following the proof of \Cref{theorem:main_result}, we add and subtract
$\ell\pli(z\pli_t; \theta\pli_t) - \ell\pli(z\pli_t; \theta\pli_\star)$ to obtain
\begin{align*}
    \underbrace{\frac{1}{T}\sum_{t=1}^T \E_{z\pli \sim \cD\pli(\theta_t)}\big[\ell\pli(z\pli;
    \theta\pli_t) - \ell\pli(z\pli; \theta\pli_\star) - (\ell\pli(z\pli_t; \theta\pli_t) -
    \ell\pli(z\pli_t; \theta\pli_\star))\big]}_{\text{small by a martingale argument}}
    + \underbrace{\frac{1}{T}\sum_{t=1}^T \big[\ell\pli(z\pli_t; \theta\pli_t) - \ell\pli(z\pli_t;
    \theta\pli_\star)\big]}_{\text{small by regret}}.
\end{align*}
Define $V\pli_t = \E_{z\pli \sim \cD\pli(\theta_t)}[\ell\pli(z\pli; \theta\pli_t) - \ell\pli(z\pli;
\theta\pli_\star)] - [\ell\pli(z\pli_t; \theta\pli_t) - \ell\pli(z\pli_t; \theta\pli_\star)]$. We show
that $V\pli_1, \dots, V\pli_T$ is a martingale difference sequence with respect to the
filtration $\{\mathcal{F}_t\}_{t \geq 0}$ defined by $\mathcal{F}_0 \defeq
\sigma(\theta_1)$ and
\begin{align*}
    \mathcal{F}_t \defeq \sigma(\theta_1, z_1, \theta_2, z_2, \dots, \theta_t, z_t,
    \theta_{t+1}) \quad \text{for } t \geq 1.
\end{align*}
By construction, $\theta_t$ (and hence $\theta\pli_t$) is
$\mathcal{F}_{t-1}$-measurable, and conditional on $\mathcal{F}_{t-1}$ the sample
$z\pli_t$ is an independent draw from $\cD\pli(\theta_t)$. Hence $V_t^i$ is
$\mathcal{F}_t$-measurable and $\E[V_t^i \mid \mathcal{F}_{t-1}] = 0$. Since losses are
bounded, $\E[|V_t^i|] < \infty$, and therefore $\E_S\big[\sum_{t=1}^T V_t^i\big] = 0$.

For the second term, the no-regret guarantee of player $i$'s algorithm gives
\begin{align*}
    \sum_{t=1}^T \ell\pli(z\pli_t; \theta\pli_t) - \ell\pli(z\pli_t; \theta\pli_\star) \leq
    \Reg\pli(T).
\end{align*}
Combining the two bounds and noting that the argument holds for any $\theta\pli_\star \in
\Theta\pli$ establishes the result for player $i$. Since $i$ was arbitrary, the bound
holds simultaneously for every player.
\end{proof}

Similarly to \cref{theorem:main_result}, the result can be extended to hold path-wise in high probability by using a concentration inequality and taking a union bound over all comparators and players.

\paragraph{Comparison to prior work.} \citet{narang2023multiplayer} established the first convergence results to multiplayer performative stability, showing that repeated retraining and gradient descent converge if every player's loss $\ell\pli$ is $\beta_i$-smooth and $\gamma_i$-strongly convex, every distribution map $\cD\pli(\cdot)$ is $\eps_i$-sensitive in the sense of \Cref{def:sensitivity}, and  $\sqrt{\sum_{i=1}^n (\beta_i \eps_i / \gamma_i)^2} < 1$. That is, the \emph{norm} across $n$ players of the individual parameters $\beta_i \eps_i / \gamma_i$ needs to be small, a condition which becomes harder to satisfy as the number of players $n$ increases.

\Cref{thm:multi} sidesteps this entirely. We require no smoothness or strong convexity on any $\ell\pli$, no sensitivity assumption on any $\cD\pli(\cdot)$, and in particular no joint condition relating constants across players. As long as each player runs a no-regret algorithm on their own observed losses, the uniform mixture over the joint trajectory is performatively stable, regardless of how many players there are or how strongly their predictions influence the shared environment.

\section{Applications to Stateful Performative Prediction}
\label{sec:stateful}

On a similar note to the previous set of results, the same analysis template can be used to understand the stateful case where data is a function of the entire history of deployed models, rather than just the currently deployed model. This setting was introduced and first studied in \citet{brown2022performative} and captures settings such as recommender systems. Consider, for instance, music platforms. Predictions of an individual's preferences are used to build recommended playlists that end up influencing their future behavior. However, this behavior depends not just on the most recent set of recommendations, but rather the entire history of deployed models. 

Here, we consider the following generalized version of the stateful performative prediction problem. Assume that the learner makes predictions sequentially over a series of rounds $t=1,\dots, T$ where at time step $t$, they predict using the model $\theta_t$. Unlike the standard setting where there is a fixed distribution map $\cD(\cdot)$ and $z_t \sim \cD(\theta_t)$, we now assume that 
\begin{align}
\label{eq:stateful_update}
    z_t \sim \cD_t.
\end{align}
Here, $\cD_t$ is an arbitrary distribution over points in $\cZ$ that is potentially a function not just of $\theta_t$, but of the entire history of models $\theta_1, \dots, \theta_t$. Unlike prior work, we \emph{make no assumptions} whatsoever on how $\cD_t$ depends on the sequence of models $\theta_1,\dots,\theta_t$. 

To make this more precise, in the initial paper on stateful performative prediction, \citet{brown2022performative} assumed that distributions were produced according to the following geometric decay dynamic,
\begin{align*}
    \cD_t = (1-\alpha) \cD_{t-1} + \alpha \cD(\theta_t),
\end{align*}
where $\cD: \Theta \rightarrow \Delta(\cZ)$ is a fixed, Lipschitz distribution map and $\alpha \in (0,1)$ is a problem parameter. In recent work, \citet{lee2026partiallyperformativeprediction} considered a different ``partially performative'' setup where distributions evolve according to the update rule: 
\begin{align*}
    \cD_t = (1-\alpha_t) \cD(\theta_t) + \alpha_t \cP_t.
\end{align*}
Here, $\cP_t$ models an exogenous distribution that is unaffected by performativity and $\cD(\cdot)$ is again a fixed, Lipschitz distribution map. Both of these are special, restricted cases of our more general setup described in \Cref{eq:stateful_update}, which can accommodate both kinds of updates.

Without any kind of restriction, the sequence of distributions $\cD_t$ may not have well-defined limits no matter the sequence of models. We therefore consider the following stability guarantee:

\begin{definition}
Assume the stateful performative prediction setup where $z_t \sim \cD_t$ and $\cD_t$ is an arbitrary distribution that can depend on the entire history of deployed models $\theta_1, \dots, \theta_t$. We say that a sequence of models $\{\theta_t\}_{t=1}^T$ is a performatively stable sequence if 
\begin{align*}
        \frac{1}{T} \sum_{t=1}^T \E_{z \sim \cD_t} [\ell(z;\theta_t)]  \leq \min_{\theta \in \Theta}\frac{1}{T} \sum_{t=1}^T \E_{z \sim \cD_t} [\ell(z ;\theta)] + o(1)
\end{align*}
where $o(1)$ denotes a term that goes to 0 as $T\rightarrow \infty$.
\end{definition}
Performatively stable sequences generalize the same desiderata from the stateless case to the stateful setting: from the perspective of retraining, there is no incentive to deviate to any alternative model $\theta$ since no fixed model has lower loss over the induced distribution.

The same reduction as before shows that no-regret algorithms produce performatively stable sequences, again without any assumptions on the dynamics. For simplicity, we only state the in-expectation version since the high-probability statement follows from an identical argument.

\begin{theorem}
\label{thm:stateful}
Let $\theta_1, \dots, \theta_T$ be the sequence of iterates produced over $T$ rounds by a deterministic online learning algorithm that guarantees sublinear regret on the sequence of losses $\ell_t(\theta) = \ell(z_t;\theta)$, where $z_t \sim \cD_t$, each $\cD_t$ is an arbitrary function of $\theta_1, \dots, \theta_t$, and $\ell$ is bounded. Let $S = \{z_t\}_{t=1}^T$. Then,
\begin{align*}
    \E_S\left[\frac{1}{T} \sum_{t=1}^T \E_{z \sim \cD_t} \ell(z;\theta_t)\right] \leq \min_{\theta \in \Theta} \E_S\left[\frac{1}{T} \sum_{t=1}^T \E_{z \sim \cD_t} \ell(z;\theta)\right] + \frac{\regret_T}{T}.
\end{align*}
Hence, $\{\theta_t\}_{t=1}^T$ is a performatively stable sequence.
\end{theorem}
\begin{proof}
Fix an arbitrary comparator $\theta_\star \in \Theta$. As in the proof of \Cref{theorem:main_result}, we add and subtract $\ell(z_t;\theta_t) - \ell(z_t;\theta_\star)$ and write
\begin{align*}
    \frac{1}{T}\sum_{t=1}^T \E_{z \sim \cD_t}\Big[\ell(z;\theta_t) - \ell(z;\theta_\star)\Big] = \frac{1}{T}\sum_{t=1}^T V_t + \frac{1}{T}\sum_{t=1}^T \Big[\ell(z_t;\theta_t) - \ell(z_t;\theta_\star)\Big],
\end{align*}
where $V_t \defeq \E_{z \sim \cD_t}[\ell(z;\theta_t) - \ell(z;\theta_\star)] - [\ell(z_t;\theta_t) - \ell(z_t;\theta_\star)]$. The second term is at most $\regret_T/T$ by the no-regret guarantee. The first is zero in expectation over the randomness in $S$ since it is again a martingale difference sequence. 
\end{proof}

\paragraph{Comparison to prior work.} Following the theme of the paper, the theorem above establishes that, under no assumptions on the distribution map, online algorithms yield a generalized notion of stability. To the best of our knowledge, this is the first assumption-free notion of stability in the stateful case. In particular, both \cite{brown2022performative} and \cite{lee2026partiallyperformativeprediction} work with the classical assumptions that losses are well-conditioned and that at least part of the underlying data generating process is sufficiently Lipschitz. Our work illustrates how one can achieve a similarly motivated notion of stability while removing all assumptions around how people respond to predictions.

\section{Discussion and Future Work}

In this work, we establish a reduction between no-regret learning in online optimization and performative stability. As a consequence of this connection, we expand the regimes of problems in which one can achieve this natural solution concept and provide an explanation for why common learning procedures do not produce runaway feedback loops in dynamic, social environments. 
We believe that there are a number of exciting directions to pursue in future work. 

One especially important direction pertains to connections between stability and optimality. As we discussed in \Cref{sec:preliminaries}, performatively stable models are fixed points of retraining, but they do not necessarily minimize the performative risk. Prior work by \cite{miller2021outside} and \cite{perdomo2025revisiting} shows that there exist worst-case instances where stable solutions in fact \emph{maximize} the performative risk. This begs a natural question: what structural conditions can we impose on $\cD(\cdot)$ to ensure that stable points are also optimal? We feel this question is particularly interesting in light of our results since we now have a robust set of algorithmic techniques to achieve stability. Finding natural settings where stability and optimality coincide would enable us to apply this algorithmic toolkit to achieve a thus far more challenging notion.

\printbibliography

\end{document}